\documentclass[letterpaper]{article} 
\usepackage{aaai24}  
\usepackage{times}  
\usepackage{helvet}  
\usepackage{courier}  
\usepackage[hyphens]{url}  
\usepackage{graphicx} 
\usepackage{natbib}  
\usepackage{caption} 
\usepackage{algorithm}
\usepackage{algorithmic}
\usepackage{amsmath,amsthm,amsfonts,amssymb}

\usepackage{color}

\input{epsf}

\newtheorem{theorem}{Theorem}

\newtheorem{remark}{Remark}
\newtheorem{definition}{Definition}
\newtheorem{assumption}{Assumption}
\usepackage{newfloat}
\usepackage{listings}
\DeclareCaptionStyle{ruled}{labelfont=normalfont,labelsep=colon,strut=off} 
\floatstyle{ruled}
\newfloat{listing}{tb}{lst}{}
\floatname{listing}{Listing}
\title{Learning the Causal Structure of Networked Dynamical Systems \\ under Latent Nodes and Structured Noise}
\author{
     Augusto Santos\textsuperscript{\rm 1}\footnote{Corresponding author.}, Diogo Rente\textsuperscript{\rm 2}, Rui Seabra\textsuperscript{\rm 2}, José M. F. Moura\textsuperscript{\rm 2}
}
\affiliations{
    \textsuperscript{\rm 1} Instituto de Telecomunicações-IT, Lisbon, Portugal\\

    \textsuperscript{\rm 2} Department of Electrical and Computer Engineering at Carnegie Mellon University, Pittsburgh, PA, USA\\
    
    augusto.santos@lx.it.pt, drobertd@andrew.cmu.edu, rvilarpo@andrew.cmu.edu, moura@andrew.cmu.edu.

}

\begin{document}

\maketitle

\begin{abstract}
This paper considers learning the hidden causal network of a linear networked dynamical system (NDS) from the time series data at some of its nodes -- partial observability. The dynamics of the NDS are driven by \emph{colored} noise that generates spurious associations across pairs of nodes, rendering the problem much harder. To address the challenge of noise correlation and partial observability, we assign to each pair of nodes a feature vector computed from the time series data of observed nodes. The feature embedding is engineered to yield structural consistency: There exists an affine hyperplane that consistently partitions the set of features, separating the feature vectors corresponding to connected pairs of nodes from those corresponding to disconnected pairs. The causal inference problem is thus addressed via clustering the designed features. We demonstrate with simple baseline supervised methods the competitive performance of the proposed causal inference mechanism under broad connectivity regimes and noise correlation levels, including a real world network.  Further, we devise novel technical guarantees of structural consistency for linear NDS under the considered regime.
\end{abstract}

\section{Introduction}

Complex Networked Dynamical Systems (NDS) are causally structured: the state of each unit or node evolves over time due to peer-to-peer interactions or local information exchange~\cite{dynamics_complex}. Examples include: brain activity~\cite{Huang_Brain,brainaugusto}; Gene Regulatory Networks~\cite{GRN0,GRN3}; or pandemics~\cite{topoepidemics,topoepidemic2}. Understanding the underlying connectivity pattern in these applications is fundamental to forecast the long term behavior of the NDS as, for example, predicting the onset of criticality~\cite{Pereira} or neurological crises~\cite{epi_pred}; ascertaining whether a strain of virus persists or dies out in the long run~\cite{singlevirus}; or for structure-informed downstream tasks, e.g., the design of mitigation measures~\cite{Net_dismantling2} in a pandemics. 

In all these applications, usually, only the time series data at some of the nodes are readily available while the underlying causal structure linking these nodes lurks underneath. The main goal of causal inference or structure identification is to map out the underlying causal architecture of a complex system, i.e., to enable consistent inference of the latent network structure from the observed data.



For the most part, causal inference is addressed in two phases: $i)$ \emph{Estimation Phase}--- a scalar value is assigned to each pair of nodes as a proxy to the pair's coupling strength or affinity (e.g., mutual information); $ii)$ \emph{Classification Phase}--- a thresholding or hypothesis testing is applied to the estimated couplings to draw the network, consistently. 

Reference~\cite{SMachado} departed from this approach and addressed the framework of partial observability (with \emph{diagonal} or uncorrelated noise), under a novel feature based paradigm. Instead of a scalar-based estimation for the pairs' affinity, a feature vector is assigned to each pair of nodes, and structure is recovered by leveraging certain identifiability properties of the engineered features in a higher-dimensional space (instead of thresholding). Desirable characteristics of this feature based approach include: $i)$ \emph{Separability} \textemdash there exists a manifold that separates the features of connected pairs from those of disconnected pairs; $ii)$ \emph{Stability}\textemdash this separation manifold should not be \emph{too sensitive} to differences in the regimes of connectivity, observability and noise correlation of the underlying complex system; and $iii)$ \emph{Locality}\textemdash the feature of each pair can be computed from the time series of each pair (or neighboring pairs thereof). \emph{Separability} is a necessary property to consistently cluster the features, i.e., for structure inference; \emph{stability} renders supervised ML tools (e.g., SVM) amenable to generalization;  and \emph{locality} of the features is crucial for large-scale systems where only a subset of nodes can be feasibly observable. 



However, in most applications, the underlying excitation noise exhibits nontrivial correlation structure. Depending on the regime of observation or noise structure, information about the main object of inference may be fundamentally lost~\cite{Barbier_noise}. The level of noise plays a major role in the feasibility of statistical inference tasks in that the problem can be \emph{impossible}, \emph{hard} (possible, but no known polynomial-time algorithm), or \emph{easy}, depending on the noise level~\cite{Barbier_Phase}. It is thus, critical to understand the impact of the noise structure on inference problems.

In this work, we address causal inference under the delicate regime of partial observability and structured noise. For this structure identification problem, and defying common wisdom in statistical physics, we show that structural information is preserved in the partially observed time series irrespective of the level of noise correlations, as long as they are sufficiently homogeneous across pairs (Theorem~$2$). Under greater correlation heterogeneity, we show that exogenous interventions help recover feasibility. 

Our contribution is twofold: $i)$ we establish a novel identifiability condition for linear networked dynamical systems (Theorem~$2$); $ii)$ we propose novel features to devise a Machine Learning based causal inference mechanism exhibiting competitive performance under the referred regime. Point $i)$ provides, in particular, novel guarantees for linear networked dynamical systems, where the information of the underlying network structure is contained in the observed time series data (i.e., it is not fundamentally lost).

\section{Related Work}

There is no universal method for causal inference, i.e., there does not exist an algorithm that is consistent and optimal across a broad range of attributes, including: 1) dynamical laws; 2) structural patterns -- sparse, dense networks; 3) regimes of observability (full or partial); and 4) statistical properties of the input or perturbation noise. For example, if the observed data follow a Gaussian multivariate distribution, the \textit{Precision} matrix -- or regularized versions thereof, e.g., Graphical Lasso -- is a structurally consistent estimator \cite{lohwainwright-2013}; if the samples follow a ferromagnetic Ising model distribution, then the \textit{correlation} matrix is a possible consistent estimator over sparse networks~\cite{Bento2009}; Chow-Liu's mutual information based seminal algorithm \cite{ChowLiu} is structurally consistent for a great family of probability distributions generating the i.i.d.~samples, but it is specialized to trees; Granger~\cite{Geigeretal15,tomo_journal_proceedings,tomo_journal,R1minusR3} is consistent for linear dynamical systems (even under partial observability), but it exhibits compromised performance under colored noise and partial observability; conditional independence tests and corresponding measures seldom fit densely connected networked systems. 

For a recent comprehensive discussion on related causal inference works, please refer to~\cite{SMachado} or~\cite{tomo_journal_proceedings}. Most works focus on sparse networks -- by methods primarily promoting model-parsimony -- and/or \emph{unstructured} excitation noise, i.e., independence of the exogenous noise across distinct nodes in the system. While this renders the inference task technically tractable, for the most part, it is an overall unrealistic assumption. For example, this noise independence across nodes is a fundamental assumption in Structural Causal Models (SCM)~\cite{pearl_2009,Elements_Causal_Inference} and underlies the structural consistency of conditional independence tests based algorithms~\cite{spirtes01a,Spirtes2000,anandkumar13,AnandkumarValluvanAOS} as well as other ML-based approaches~\cite{Causal_Transformer,Irregular_Time_Series_Causal,Earth_Systems_Causation,CUTS_Neural_Causal,Granger_Causal_Interpretable}. Several works explore linear networked dynamical systems~\cite{MaterassiSalapakaCDC2012,MaterassiSalapakaCDC2015,NIPS_Bento,tomo_journal,open_Journal} tailored to sparse networks and uncorrelated noise. The recent reference~\cite{Salapaka_Colored} considers linear dynamic influence models (LDIMs), defined in the frequency domain, with correlated noise and develops a Wiener based approach specialized to sparse networks -- leveraging on a sparse $+$ low rank method. The reference makes  important limiting assumptions on the  noise correlation: $i)$~it results from affine interactions; and, $ii)$~the  excitation noise is assumed independent across latent nodes.    

In this work, we address the causal inference problem for linear networked dynamical systems (NDS) when the noise across nodes is correlated and under the partial observability regime (only a subset of nodes is observed). We provide  novel technical guarantees of structural consistency or identifiability and show that the proposed approach exhibits significant superior performance than benchmark alternatives across distinct regimes of network connectivity, observability, and noise correlation. We assume no knowledge of the underlying noise covariance matrix or the network of interactions, and we impose no assumptions on the particular nature of the noise correlations.

\begin{figure} [!t]
	\begin{center}	\includegraphics[scale= 0.33]{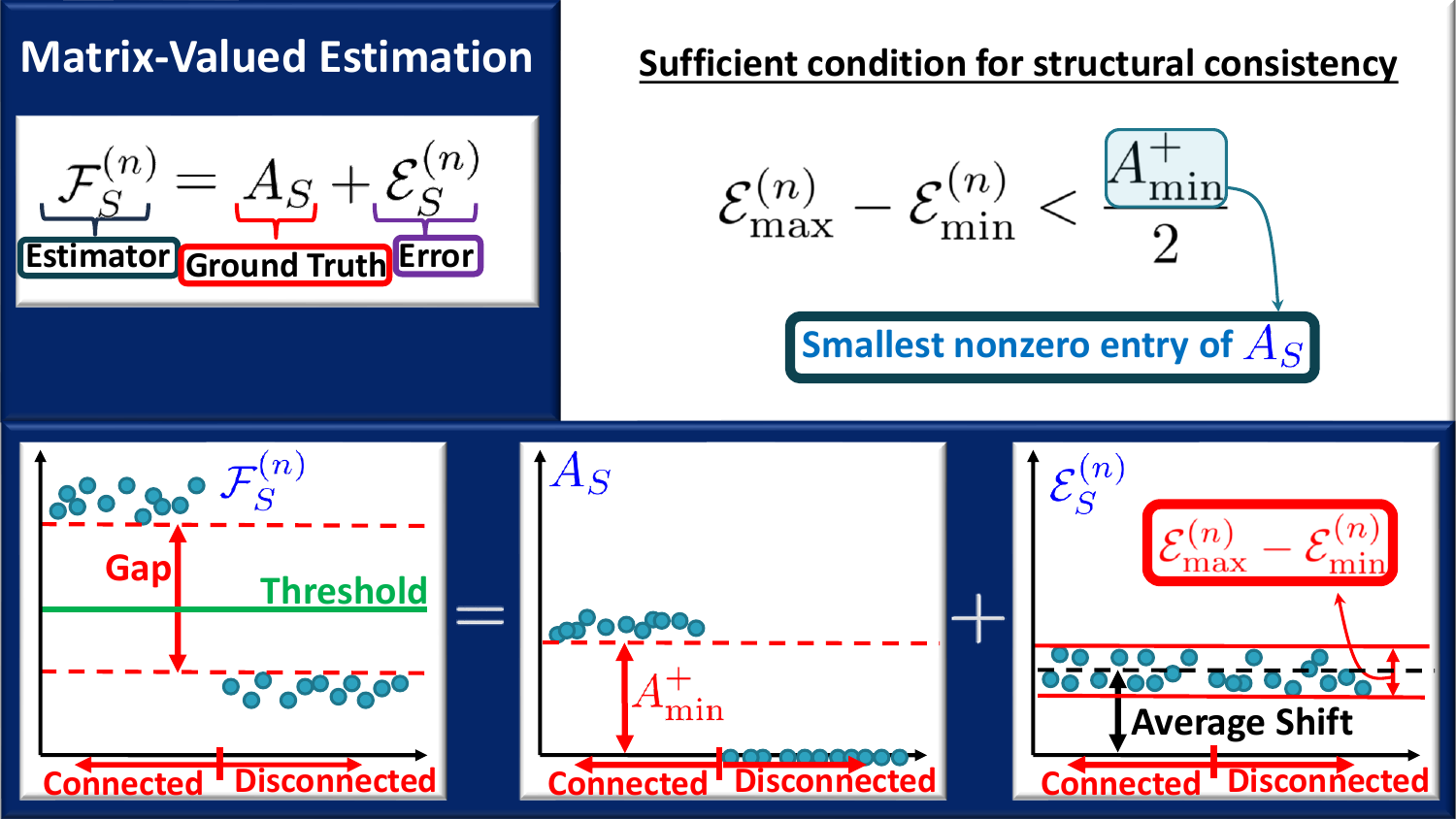} 
		\caption{Sufficient condition for structural consistency. The plots represent the off-diagonal entries of the corresponding matrices: the estimator, the ground truth and the error.}
		\label{fig:SC}
	\end{center}
\end{figure}


\section{Problem Formulation}
\label{sec:probform}


The time series data streams from the linear Networked Dynamical System (NDS) given by
\begin{equation}\label{eq:model}
\mathbf{y}(n+1)=A\mathbf{y}(n)+\mathbf{x}(n+1),
\end{equation}
where~$\mathbf{y}(n)=\left[y_1(n)\,\,y_2(n)\,\,\ldots\,\,y_N(n)\right]^{\top}\in\mathbb{R}^N$ represents the state-vector of the $N$-dimensional NDS collecting the state~$y_i(n)$ of each node~$i=1,2,\ldots,N$ at time $n$; $\mathbf{x}(n)$ is the zero mean input noise at time~$n$ with finite first and second order moments, i.e., it admits a well-defined covariance matrix~$\Sigma_x\in\mathbb{S}_+^N$, and $\left(\mathbf{x}(n)\right)_{n\in\mathbb{N}}$ is independent across time~$n$; $A\in \mathbb{R}_{+}^{N}$ is the non-negative symmetric interaction matrix describing the graph linking the nodes as its support. We assume that the spectral radius of~$A$ is smaller than~$1$, $\rho(A)<1$, i.e., the dynamical system is assumed stable. Namely, we assume that $A$ is stochastic up to a multiplicative constant, i.e., $A=\rho \overline{A}$ where $\overline{A}$ is stochastic and $\rho<1$.


The discrete-time linear NDS~\eqref{eq:model} emerges naturally from time discretization of continuous-time linear stochastic differential equations as, e.g., in~\cite{NIPS_Bento,GRN0}. These linear stochastic networked dynamical systems may also describe nonlinear networked dynamical systems operating close to an equilibrium state with \emph{small} enough noise level~\cite{Ching2017ReconstructingLI, Noise_Jien_Ren}, in the spirit of the Hartman-Grobman linearization Theorem~\cite{Barreira}. Linearity can also be recovered via appropriate Koopman-based embedding in higher dimensional spaces~\cite{Koopman_Gon,Nonlinear_to_linear,Koopman_Linear}, change of variables, or statistical transformations like Box-Cox~\cite{BoxCox2}. These render structural inference methods developed for the linear NDS~\eqref{eq:model} consistent or useful over a broad class of nonlinear NDS.

The goal of graph recovery under \emph{full} observability is to infer consistently the support of the interaction matrix~$A$ from the time series data $\left\{\mathbf{y}(n)\right\}_{n=1}^{T}$. Under partial observability~\cite{SMachado,tomo_journal,Mei_latent,tomo_journal_proceedings,Geigeretal15,open_Journal}, it is to recover the support of the submatrix $A_{S}$ that links a set $S:=\left\{m_1,m_2,\ldots,m_{\left|S\right|}\right\}$ of observable nodes from their time series  $\left\{\left[\mathbf{y}(n)\right]_S\right\}_{n=1}^T$, where $\left[\mathbf{y}(n)\right]_S\overset{\Delta}= \left[\mathbf{y}_{m_1}(n)\,\,\, \mathbf{y}_{m_2}(n)\cdots \mathbf{y}_{m_{\left|S\right|}}(n)\right]^{\top} \in\mathbb{R}^{\left|S\right|}$
is the $\left|S\right|$-dimensional vector collecting the states of the observed nodes at each time $n$, $\left|S\right|<N$. 


Define the $k^{\text{th}}$ lag empirical covariance matrix $\widehat{R}_k(n)\overset{\Delta}=\frac{1}{n}\sum_{\ell=0}^{n-1} \mathbf{y}(\ell+k)\mathbf{y}(\ell)^{\top}$.
Similarly to as done in~\cite{SMachado}, we refer to a \textit{matrix-valued estimator} as any map whose input is given by the (observed) time series data and the output is given by a matrix, namely,
\begin{equation}\label{eqn:mtrxestimator-1}
\begin{array}{cccc}
F^{(n)}_S\,:\, & \mathbb{R}^{\left|S\right|\times n} & \longrightarrow & \mathbb{R}^{\left|S\right|\times \left|S\right|}\\
       & \left\{\left[\mathbf{y}(\ell)\right]_{S}\right\}_{\ell=0}^{n-1} & \longmapsto & \mathcal{F}^{(n)}_S
\end{array},
\end{equation}
for any given $n\in\mathbb{N}$. The $ij^{\text{th}}$ entry of the matrix~$\mathcal{F}^{(n)}_S$ reflects the strength of the edge from $j$ to $i$ from $n$ time series samples. That is, $\mathcal{F}^{(n)}_S$ is a matrix summarizing the estimated affinities between the pairs of nodes in the system. The estimator $F^{(n)}_S$ is \emph{structurally consistent} whenever the estimated strength $\mathcal{F}_{ij}^{(n)}$ of any connected pair $(i,j)$ lies above the estimated strength of any disconnected pair. In this case, the underlying network structure is recovered via proper thresholding. This is illustrated in Fig.~\ref{fig:SC} and it is formalized with the following definition presented in~\cite{SMachado}.


\begin{definition}[Structural Consistency]\label{def:structuralconsistency}
A matrix-valued estimator~$F^{(n)}$ is structurally consistent with high probability, whenever there exists a threshold~$\tau$ so that,
$\mathbb{P}\left(\mathcal{F}^{(n)}_{ij} > \tau\right)\overset{n\rightarrow \infty}\longrightarrow 1 \Longleftrightarrow j\rightarrow i$, i.e., $j$ links to $i$ if and only if the $ij^{\text{th}}$ entry of the estimator matrix $\mathcal{F}^{(n)}$ lies above the threshold $\tau$, given a large enough number of samples~$n$.
\end{definition}

That is, up to a proper threshold~$\tau\in \mathbb{R}$, the output matrix~$\mathcal{F}^{(n)}$ reflects the underlying structure of the graph in that $\left[{\sf Supp}(A_{S})\right]_{ij}=\mathbf{1}_{\left\{\mathcal{F}^{(n)}_{ij}>\tau\right\}}$, for all pairs $i\neq j$ w.h.p.

We adopt the following three assumptions throughout.

\begin{assumption} [NDS Stability] The interaction matrix $A$ is assumed symmetric, nonnegative and stable, i.e., $\rho(A)<1$, where $\rho(A)$ is the spectral radius of $A$.
\end{assumption}

\begin{assumption} [Nodewise Homogeneity] We assume that $\sigma^2=\mathbb{E}\left[\mathbf{x}_i^2\right]$ for all $i$. 
\end{assumption}

Remark that $\sigma^2\geq \mathbb{E}\left[\mathbf{x}_i\mathbf{x}_j\right]$ $\forall{i,j}$, necessarily. Indeed, 
\begin{equation}\label{ineq:heterogeneity}
\begin{array}{ccl}
0\leq E[\left(\mathbf{x}_{i}-\mathbf{x}_k\right)^2] \!\!& \!\! = \!\!& \!\! E[\mathbf{x}_{i}^2]+E[\mathbf{x}_k^2]-2E[\mathbf{x}_{i} \mathbf{x}_k]\\
\!\!& \!\! =\!\!& \!\! 2\sigma^2-2E[\mathbf{x}_{i} \mathbf{x}_k].
\end{array}  
\end{equation}

\begin{assumption}[Pairwise Distinguishability] We assume that $\sigma^2>\mathbb{E}\left[\mathbf{x}_i \mathbf{x}_j\right]=\left[\Sigma_{x}\right]_{ij}$ for all $i\neq j$. That is, the off-diagonal entries of the noise covariance matrix are strictly smaller than the diagonal.
\end{assumption}

Assumption~$1$ makes the NDS stable and it is thus natural. Assumption~$2$ implies Assumption~$3$ whenever there are no pair of nodes with the same noise source, i.e., $\mathbb{P}\left(\mathbf{x}_i\neq \mathbf{x}_j\right)>0$, for $i\neq j$. Indeed, from~\eqref{ineq:heterogeneity}, $\sigma^2=E\left[\mathbf{x}_{i}\mathbf{x}_k\right]$ if and only if $\mathbb{E}\left[\left(\mathbf{x}_i-\mathbf{x}_k\right)^2\right]=0$ which implies $\mathbf{x}_i\overset{{\sf a.s.}}=\mathbf{x}_k$. Therefore, Assumption~$3$ is mild under Assumption~$2$. 

In view of Assumptions~$2$ and $3$, we shall refer to 
$ \sigma^2_{{\sf gap}}\overset{\Delta}=\sigma^2-\max_{i\neq j} \mathbb{E}\left[\mathbf{x}_i\mathbf{x}_j\right]>0$
as the gap between the diagonal and the off-diagonal entries of the noise covariance $\Sigma_x$.

\section{Structural Consistency under Noise Structure and Partial Observability}

We write a matrix-valued estimator, see equation~\eqref{eqn:mtrxestimator-1},  as
$\mathcal{F}_S^{(n)}=A_{S}+\mathcal{E}^{(n)}_S$
where, under partial observability, $A_S$ is the underlying ground truth interaction matrix, $\mathcal{E}^{(n)}_S$ is the (matrix) error term, and $S$ is the set of observable nodes. As we now show, for structural consistency, we do not need the error term to be zero (or to converge to zero with~$n$). Define ${\sf Osc}\left(\mathcal{E}^{(n)}_S\right)\overset{\Delta}=\mathcal{E}^{(n)}_{\max}-\mathcal{E}^{(n)}_{\min}$ and let $A^+_{\min}$ be the smallest nonzero entry of $A_S$,
where $\mathcal{E}^{(n)}_{\max}$ and $\mathcal{E}^{(n)}_{\min}$ are the maximum and minimum off-diagonal entries of the error matrix $\mathcal{E}_S^{(n)}$. Then, if, with high probability (w.h.p.),
\begin{equation}\label{ineq:separab}
{\sf Osc}\left(\mathcal{E}^{(n)}_S\right)\leq \frac{A^{+}_{\min}}{2},
\end{equation}
it is easily seen that $F^{(n)}_S$ is structurally consistent. In fact, condition~\eqref{ineq:separab} precludes \emph{inversion} and consequent loss of structural information in Definition~\ref{def:structuralconsistency} in the following sense: If $ij$ is a connected pair and $k\ell$ is a disconnected pair, then constraint~\eqref{ineq:separab} implies that $\mathcal{F}_{ij}^{(n)}> \mathcal{F}^{(n)}_{k\ell}$ necessarily. This means that, under~\eqref{ineq:separab}, structure can be recovered via properly thresholding the off-diagonal entries of the matrix~$\mathcal{F}^{(n)}_S$, otherwise, structural information is lost. Further, when the error matrix is \emph{flat}, i.e., ${\sf Off}\left(\mathcal{E}_S^{(n)}\right)\propto {\sf Off}\left(\mathbf{1}\mathbf{1}^{\top}\right)$, all the entries of $F^{(n)}$ are the entries of $A_S$ shifted by the same quantity. Therefore, again, structural information is preserved. 
Fig.~\ref{fig:SC} offers an illustration of this discussion.

\textbf{Example: Granger estimator.} Reference~\cite{tomo_journal_proceedings} studied the error $\mathcal{E}_{S}$ for the Granger estimator $\left[R_1\right]_S\left(\left[R_0\right]_S\right)^{-1}$, under partial observability, where $R_1$ and $R_0$ are the limit one-lag and correlation matrices, respectively. Under full observability, Granger recovers the underlying interaction matrix~$A$ w.h.p.~since $R_1\left(R_0\right)^{-1}=A$. In other words, the limiting error matrix $\mathcal{E}$ associated with Granger under full observability is zero (trivially flat). However, under partial observability, the limiting error matrix associated with Granger is given by~\cite{tomo_journal_proceedings}
$\mathcal{E}_S = A_{SS'}\left(\left[R_0^{-1}\right]_{S'}\right)^{-1}\left[A^2\right]_{S'S}$, 
where $S'$ is the set of latent nodes (complement of $S$). For diagonal and correlated noise structures, this leads to 
\begin{equation}
\begin{array}{lcl}
 \mathcal{E}^{({\sf diagonal})}_S & = & A_{SS'}\left(I_{S'}-\left[A^2\right]_{S'}\right)^{-1}\left[A^2\right]_{S'S}\\
\mathcal{E}_S^{({\sf colored})} & = &  A_{SS'}\left(\left[(\sum_{i=0}^{\infty} A^i \Sigma_x A^i)^{-1}\right]_{S'}\right)^{-1}\left[A^2\right]_{S'S}
 \end{array}.
\end{equation}
The error $\mathcal{E}_S^{({\sf diagonal})}$ has been thoroughly studied in~\cite{tomo_journal_proceedings} that shows structural consistency of the Granger estimator under partial observability and diagonal noise. However, even for Granger, there is no known analysis of $\mathcal{E}_S^{({\sf colored})}$. In fact, under  colored noise and partial observability, Granger's performance degrades.\hfill$\small \square$

In summary, the behavior of the error matrix~$\mathcal{E}_S^{(n)}$ determines  whether structural information is preserved in the observed time series or fundamentally lost: an estimator $F^{(n)}_S$ is structurally consistent whenever the error matrix is flat enough w.h.p. However, given a particular estimator, this error term responds differently to distinct regimes of $i)$~connectivity of the underlying networked dynamical system, $ii)$~observability, or $iii)$~noise structure. This discussion shows the importance of two delicate steps in structure estimation: $i)$~the characterization of the error term; and $ii)$~the analysis of the \emph{flatness} of the error. Indeed, what matters is the \textit{variability} and not the (average) drift of its entries. It also alludes to the difficulty in handling the noise structure.

This paper addresses these questions. The next section characterizes the error term for the estimator $\widehat{R}_1(n)-\widehat{R}_3(n)$ as a function of the covariance matrix~$\Sigma_x$. The following section  establishes a sufficient condition on $\Sigma_x$ to guarantee that the error is flat enough, and so guaranteeing w.h.p.~structural consistency of $\widehat{R}_1(n)-\widehat{R}_3(n)$. This provides a sufficient condition that fundamentally asserts when structural information is preserved in the observed time series data in the regime of partial observability and colored noise.   


\section{Structural Consistency: Error Characterization}
The estimator~$\widehat{R}_1(n)-\widehat{R}_3(n)$ was first  proposed and studied recently in~\cite{R1minusR3}  under diagonal noise. In contrast, we characterize the limiting error term of this estimator under very broad conditions of full and partial observability and correlated noise, see 
Theorem~\ref{th:representation} below. This will be used in the next section to develop novel conditions of structural consistency under partial observability and colored noise for the networked dynamical system~\eqref{eq:model}. These are  conditions on the covariance matrix~$\Sigma_x$ that if verified guarantee that structural information is preserved and can be recovered  from the observed time series.   

Before proceeding to the main result, observe that, given a covariance matrix $\Sigma_x\in\mathbb{S}_{+}^{N}$ satisfying Assumptions~$2$ and~$3$, we can uniquely represent it as
\begin{equation}\label{eq:reprdecomp}
\Sigma_x:=\sigma^2_{{\sf gap}}I+\beta\mathbf{1}\mathbf{1}^{\top}+\overline{\Sigma} 
\end{equation}
where we have defined $\sigma^2_{{\sf gap}} \overset{\Delta}= \sigma^2-\max_{i\neq j} \mathbb{E}\left[\mathbf{x}_i\mathbf{x}_j\right]>0$ and $\beta \overset{\Delta}= \frac{1}{N(N-1)}\mathbf{1}^{\top}{\sf Off}(\Sigma_x)$,
with ${\sf Off}(\Sigma_x)$ representing the vector collecting the off-diagonal entries of the matrix $\Sigma_x$. Thus, $\beta$ stands for the average of the off-diagonals of $\Sigma_x$. 

The representation in equation~\eqref{eq:reprdecomp} decomposes the covariance into: $i)$~a diagonal matrix $\sigma^2_{{\sf gap}}I$; $ii)$~the average \emph{offset} matrix $\beta\mathbf{1}\mathbf{1}^{\top}$; and $iii)$~the matrix $\overline{\Sigma}$ containing the variability of the off-diagonal entries of $\Sigma_x$. This will be useful.



\begin{theorem}[Error Characterization]\label{th:representation}
Under Assumptions~$1$, $2$, and~$3$, for the NDS~\eqref{eq:model}, we have
\begin{equation}\label{eq:color}
\frac{1}{\sigma_{\sf gap}^2}\left(\widehat{R}_1(n)-\widehat{R}_3(n)\right) \overset{n\rightarrow \infty}\longrightarrow  A +\mathcal{E},
\end{equation}
with
\begin{equation}\label{eq:mathcalE-1}
\mathcal{E}\overset{\Delta}=\frac{1}{\sigma_{\sf gap}^2}\left[\beta\rho\mathbf{1}\mathbf{1}^{\top}+\left(I-A^2\right)  \left(\sum_{i=0}^{\infty} A^{i+1} \overline{\Sigma} A^i\right)\right],
\end{equation}
where the convergence holds in probability. Accordingly, under partial observability, we have
\begin{equation}\label{eq:partial}
\frac{1}{\sigma_{\sf gap}^2}\left(\left[\widehat{R}_1(n)\right]_S-\left[\widehat{R}_3(n)\right]_S\right) \overset{n\rightarrow \infty}\longrightarrow A_S + \mathcal{E}_{S}.
\end{equation}
\end{theorem}
\begin{proof}
Equation~\eqref{eq:color} is entrywise, so~\eqref{eq:partial} follows directly from~\eqref{eq:color}. We prove~\eqref{eq:color}. 
Recall that the limit covariance matrices are given in terms of the interaction matrix $A$ as $R_0\overset{\Delta}=\lim\limits_{n\rightarrow \infty}R_0(n)=\sum_{i=0}^{\infty} A^i \Sigma_x A^i$  (solution to the Lyapunov equation) and $R_k=A^k R_0$ for all $k\in\mathbb{N}$. 
These expressions for $R_0$ and $R_k$, with representation~\eqref{eq:reprdecomp}, lead to
\begin{equation}
\left\{\begin{array}{ccl} R_1 & = & \sigma_{\sf gap}^2 \sum_{i=0}^{\infty} A^{2i+1} + \beta \sum_{i=0}^{\infty} A^{i+1} \mathbf{1}\mathbf{1}^{\top} A^i \\
& & \hspace{2.53cm}+  \sum_{i=0}^{\infty} A^{i+1} \overline{\Sigma} A^i\\
&&\\
R_3 & = &  \sigma_{\sf gap}^2 \sum_{i=0}^{\infty} A^{2i+3} + \beta \sum_{i=0}^{\infty} A^{i+3} \mathbf{1}\mathbf{1}^{\top} A^i\\
& &  \hspace{2.53cm}+\sum_{i=0}^{\infty} A^{i+3} \overline{\Sigma} A^i,
\end{array}.\right.
\end{equation}
Now, $\sigma_{\sf gap}^2 \sum_{i=0}^{\infty} A^{2i+1}-\sigma_{\sf gap}^2 \sum_{i=0}^{\infty} A^{2i+3}=\sigma_{\sf gap}^2 A$ and 
\begin{align*}
&\beta \left(\sum_{i=0}^{\infty} A^{i+1} \mathbf{1}\mathbf{1}^{\top} A^i- \sum_{i=0}^{\infty} A^{i+3} \mathbf{1}\mathbf{1}^{\top} A^i\right) \\
&=\beta \left(I-A^2\right)  \left(\sum_{i=0}^{\infty} A^{i+1} \mathbf{1}\mathbf{1}^{\top} A^i\right) \\
& =\beta \left(I-A^2\right) A \mathbf{1}\mathbf{1}^{\top} (\sum_{i=0}^{\infty} \rho^{2i})
= \frac{\beta \rho (1-\rho^2)}{(1-\rho^2)}\mathbf{1}\mathbf{1}^{\top}\!\!=\beta \rho \mathbf{1}\mathbf{1}^{\top}.
\end{align*}

Finally, the last terms in $R_1$ and $R_3$ lead to the second term in~\eqref{eq:mathcalE-1}. This
shows the limits~\eqref{eq:color} (for full observability) and~\eqref{eq:partial} (for partial observability). 

Define $R_k=\lim\limits_{n\rightarrow\infty}R_k(n)=\lim\limits_{n\rightarrow\infty}\mathbb{E}\left[\mathbf{y}(n+k)\mathbf{y}(n)^\top\right]$. By modifying the proof for diagonal noise $\Sigma_x=\sigma^2 I$ in~\cite{R1minusR3}, we can show that, for general (spatially) correlated noise $\Sigma_x$, but independent across time, the empirical covariances  $\widehat{R}_1(n)$ or $\widehat{R}_3(n)$ converge in probability to the exact limit covariances $R_1$ and $R_3$.
\end{proof}


Theorem~\ref{th:representation} characterizes the error term in terms of the covariance matrix~$\Sigma_x$. As a Corollary to Theorem~$1$, if~$\Sigma_x$ is \emph{flat}, i.e., $\Sigma_x=\sigma^{2}_{\sf gap}I+\beta\mathbf{1}\mathbf{1}^{\top}$, then the error term associated with the estimator $\widehat{R}_1(n)-\widehat{R}_3(n)$ is flat as well, implying that the estimator is structurally consistent w.h.p. 



Theorem~\ref{th:representation} further offers an intuitive interpretation for the impact of the noise structure~$\Sigma_x$ on the causal inference problem. In particular, the hyperparameter~$\beta$ controls the drift or shift that the underlying ground truth matrix $A_S$ suffers under the estimation $\widehat{R}_1(n)-\widehat{R}_3(n)$. 
On the other hand, the \emph{oscillation} of the off diagonal entries given in~$\overline{\Sigma}$ implies a perturbation on the entries of $A_S$, possibly compromising structural consistency. Fig.~\ref{fig:color} offers an illustration.

\textbf{Features separability.} Lemma $1$ in~\cite{SMachado} establishes that structural consistency of $\widehat{R}_1(n)-\widehat{R}_3(n)$ implies that the set of features~$\left\{\mathcal{F}_{ij}\right\}_{i\neq j}$ 
\begin{equation}\label{eq:features}
        \mathcal{F}_{ij}(n)=\left(\left[R_D(n)\right]_{ij},\left[R_{D+1}(n)\right]_{ij},\ldots,\left[R_M(n)\right]_{ij}\right),
\end{equation}
with $D\leq 1$ and $M\geq 3$ are affinely separable, i.e., there exists a particular hyperplane (up to a shift) that consistently partitions this set. 
The previous discussion translates into the feature domain (in view of Lemma~$1$ in~\cite{SMachado}): $\beta$ controls the \emph{drift} of the features separating hyperplane (compromising stability), while the \emph{variation}~$\overline{\Sigma}$ of the off-diagonal entries of~$\Sigma_x$ induces a spread of the features. 


\section{Identifiability under Noise Structure}




The next theorem provides a sufficient condition on the noise covariance $\Sigma_x$ and the interaction matrix $A$ that lead to the linear separability and stability of the (centered) features. 

Before proceeding to our main theorem, define ${\sf Osc}(\mathbf{z})=\max_i z_i- \min_i z_i$ as the difference between the maximum and minimum entries of the vector $\mathbf{z}$ -- we shall loosely refer to it as \emph{oscillation}. We state some properties of the map ${\sf Osc}\,:\,\mathbb{R}^{N}\longrightarrow \mathbb{R}_{+}$ that will be useful to establish the result.


\textbf{Property 1. [Contraction]} Given a symmetric stochastic matrix $\overline{A}$, we have
${\sf Osc}\left(\overline{A} v\right)\leq {\sf Osc}\left(v\right)$,
as each entry of $\overline{v}\overset{\Delta}=\overline{A}v$ lies in the convex hull~\cite{convex} of the set $\left\{v_1,v_2,\ldots,v_N\right\}$ of the entries of $v\in \mathbb{R}^N$ and, in particular, $\overline{v}_i\in \left[v_{\min},v_{\max}\right]$ for all $i$.


\textbf{Property 2. [Scalar linearity]} Observe that ${\sf Osc}\left(\alpha v\right)=\left|\alpha\right| {\sf Osc}\left(v\right)$ for any $v\in \mathbb{R}^N$ and $\alpha\in\mathbb{R}$. Indeed, if $\alpha>0$, then we have $\alpha v_{\max}>\alpha v_i$ for any $i=1,2,\ldots,N$.


\textbf{Property 3. [Subadditivity]} The map ${\sf Osc}(\cdot)$ obeys
${\sf Osc}\left(B+C\right)\leq {\sf Osc}\left(B\right)+ {\sf Osc}\left(C\right)$,
for any $B,C\in\mathbb{R}^{N}$.


\textbf{Property 4. [Submultiplicativity]} The map ${\sf Osc}(\cdot)$ obeys: If ${\sf Osc}\left(B v\right) \leq k_b {\sf Osc}\left(v\right)$ and ${\sf Osc}\left(C v\right) \leq  k_c {\sf Osc}\left(v\right)$ for all $v \in \mathbb{R}^N$
then,${\sf Osc}\left(CB v\right)\leq k_c {\sf Osc}\left(Bv\right) \leq k_bk_c {\sf Osc}\left(v\right)$
with $k_b,k_c>0$.

Let~$\mathcal{C}\left(\mathcal{K}\right)\overset{\Delta}=\frac{1}{\left|\mathcal{K}\right|}\sum_{k\in\mathcal{K}} k$ be the centroid of the set $\mathcal{K}$.

\begin{theorem}[Linear Separability \& Stability]\label{th:affine}
    Let $A=\rho \overline{A}$, where $\overline{A}$ is a stochastic matrix with $0<\rho<1$. Under Assumptions~$2$ and $3$, if
    \begin{equation}\label{eq:exogeneous}
        \frac{{\sf Osc}\left({\sf Off}(\Sigma_x)\right)}{\sigma_{{\sf gap}}^2}\leq \frac{A^{+}_{\min}(1-\rho^2)}{2\rho(\rho^2+1)},
    \end{equation}
    where $\sigma^2_{{\sf gap}}\overset{\Delta}=\sigma^2-\max_{i\neq j} \mathbb{E}\left[\mathbf{x}_i\mathbf{x}_j\right]>0$ and $A^{+}_{\min}$ is the smallest nonzero entry of the interaction matrix~$A$, then the centered features $\left\{\mathcal{F}_{ij}(n)\right\}_{i\neq j}-\mathcal{C}\left(\left\{\mathcal{F}_{ij}(n)\right\}_{i\neq j}\right)$ are linearly separable and stable.
\end{theorem}

\begin{figure} [!t]
	\begin{center}		\includegraphics[scale= 0.38]{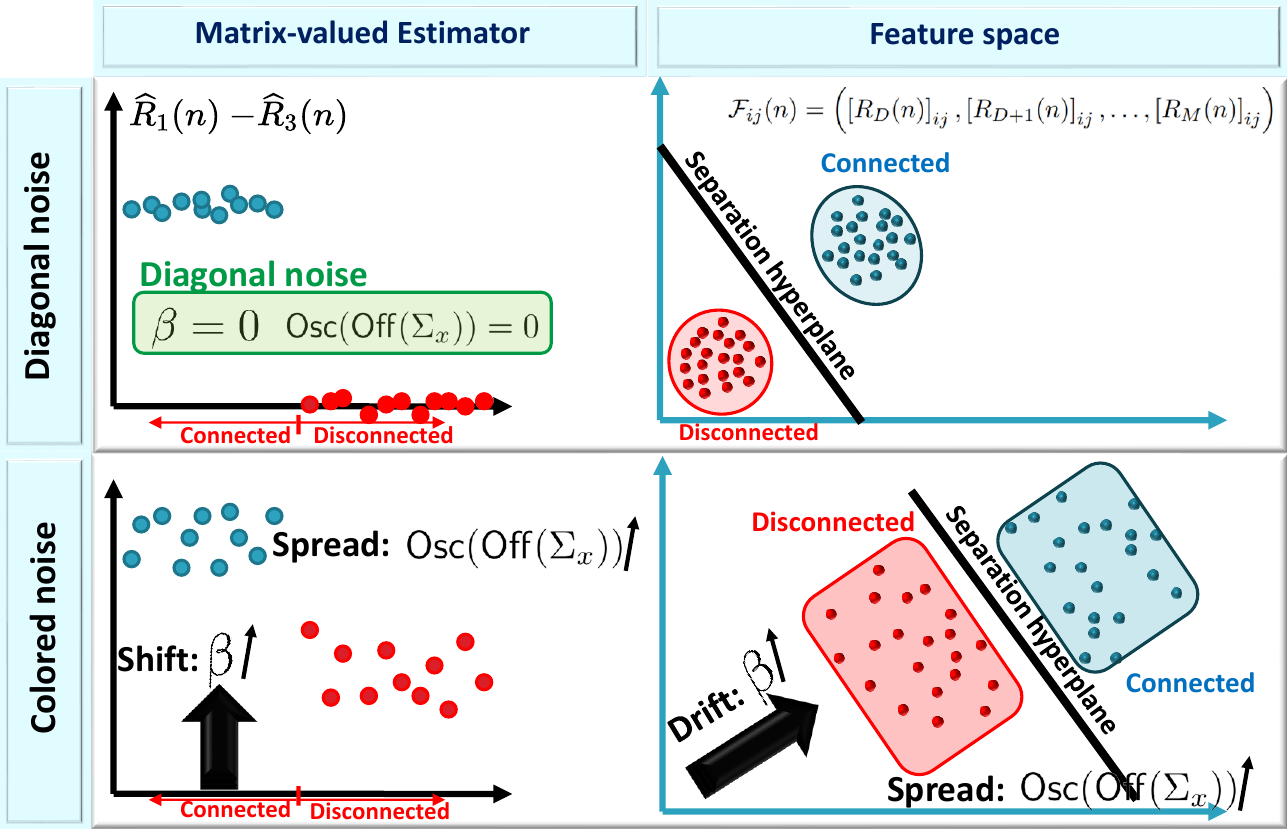} 
		\caption{Impact of the correlation structure of the noise on the features. The overarching idea (Theorem~$1$) is that the average~$\beta$ across the off-diagonals of the covariance $\Sigma_x$ yields a drift of the features, while the \emph{oscillation} across the off-diagonals of $\Sigma_x$ leads to a greater \emph{spread} of the features.}
		\label{fig:color}
	\end{center}
\end{figure}

\begin{proof}
The properties of the map ${\sf Osc}(\cdot)$ yield
\begin{equation}
\begin{array}{ccl}
{\sf Osc}\left(\left(I-A^2\right) V\right) & \leq & {\sf Osc}\left(V\right)+{\sf Osc}\left(A^2 V\right) \\
& \leq & {\sf Osc}(V)+\rho^2 {\sf Osc}\left(\overline{A}^2 V\right)\\ & \leq & \left(1+\rho^2\right){\sf Osc}\left(V\right)
\end{array},
\end{equation}
where the first inequality holds in view of Properties~$2$ and~$3$; the second inequality follows from Property~$3$; and the last inequality stems from 
Properties~$1$ and~$4$. Therefore, we have
\begin{equation}
\begin{array}{ccl}
{\sf Osc}\left(\left(I-A^2\right)A V\right) & = & \rho\, {\sf Osc}\left(\left(I-A^2\right)\overline{A} V\right) \\
& \leq & \rho\left(1+\rho^2\right) {\sf Osc}(V).
\end{array}
\end{equation}
Now, observe that $\sum_{i=0}^{\infty} A^i \overline{\Sigma} A^i$ is bounded above and below as follows
\begin{align*}
&\frac{\overline{\Sigma}_{\min}}{1-\rho^2}\mathbf{1}\mathbf{1}^{\top}  =  \overline{\Sigma}_{\min}\sum_{i=0}^{\infty} A^i \mathbf{1}\mathbf{1}^{\top} A^i
 \leq \sum_{i=0}^{\infty} A^i \overline{\Sigma} A^i \leq  \\
&\overline{\Sigma}_{\max}\sum_{i=0}^{\infty} A^i \mathbf{1}\mathbf{1}^{\top} A^i  =  \frac{\overline{\Sigma}_{\max}}{1-\rho^2}\mathbf{1}\mathbf{1}^{\top},
\end{align*}
and thus,${\sf Osc}\left(\sum_{i=0}^{\infty} A^i \overline{\Sigma} A^i\right) \leq \frac{\left(\overline{\Sigma}_{\max}-\overline{\Sigma}_{\min}\right)}{1-\rho^2}$.
Therefore,
\begin{equation}\label{ineq:imp}
\begin{array}{ccl}
{\sf Osc}\left(\mathcal{E}\right) & = & {\sf Osc}\left(\left(I-A^2\right)A\sum_{i=0}^{\infty} A^i \overline{\Sigma} A^i\right)\\
& \leq & \frac{\rho(1+\rho^2)}{1-\rho^2} \left(\overline{\Sigma}_{\max}-\overline{\Sigma}_{\min}\right)\\
& = & \frac{\rho(1+\rho^2)}{1-\rho^2}{\sf Osc}\left({\sf Off}(\overline{\Sigma}_x)\right).
\end{array}
\end{equation}
In view, of inequality~\eqref{ineq:separab} and Theorem~$1$, we have structural consistency (or features separability) whenever ${\sf Osc}\left(\mathcal{E}\right)\leq \frac{\sigma^2_{\sf gap} A^{+}_{\min}}{2}$,
which holds whenever $\frac{\rho(1+\rho^2)}{1-\rho^2}{\sf Osc}\left({\sf Off}(\overline{\Sigma}_x)\right)\leq \frac{\sigma^2_{\sf gap} A^{+}_{\min}}{2}$,
in view of inequality~\eqref{ineq:imp}.
\end{proof}
 

\begin{remark}[Exogenous interventions]\label{rem:ex} Under exogenous interventions to the NDS, i.e., $\mathbf{y}(n+1)=A\mathbf{y}(n)+\mathbf{x}(n+1)+\mathbf{\xi}(n+1),$
where $\xi(n)\sim \mathcal{N}(0,\sigma_{\xi}^2)$ is i.i.d., and independent of the process $\left(\mathbf{x}(n)\right)$, equation~\eqref{eq:exogeneous} becomes $\frac{{\sf Osc}\left({\sf Off}(\Sigma_x)\right)}{\sigma_{{\sf gap}}^2+\sigma_{\xi}^2}\leq \frac{A_{\min}(1-\rho^2)}{2\rho(\rho^2+1)},$
implying that regardless of the covariance matrix of the input process $\left(\mathbf{x}(n)\right)$, the features are linearly separable w.h.p. as soon as the level (or variance) $\sigma_{\xi}^2$ of the external intervention $\left(\xi(n)\right)$ is high enough. 
\end{remark}


We recall that due to the sensitivity of the separating hyperplane on the color of the noise, the method deployed in~\cite{SMachado} degrades as the average of the absolute values of the off-diagonal entries of $\Sigma_x$ increases (captured by $\beta$). To mitigate the shift of the separating hyperplane and allow the deployment of supervised methods, we renormalize the features applying \emph{standard scaler}. Further, we observed empirically that  incorporating novel matrix-valued estimators into the vector of features~$\mathcal{F}_{ij}(n)$ of each pair $ij$, greatly helps to mitigate both the drift and the \emph{oscillation} in the features. This is introduced next.   

\section{Novel Features}
Define the set of features $\mathcal{T}^{(n)}=\left\{\mathcal{T}^{(n)}_{ij}\right\}_{ij}$
\begin{equation}\label{eq:featurestau-1}
 \mathcal{T}^{(n)}_{ij} \!=\!       \left(\!\!\left[\!\!\left(\left[\widehat{R}_0(n)\right]_{S}\right)^{-1}\right]_{ij}\!\!,\ldots,\!\!\left[\left(\left[\widehat{R}_M(n)\right]_{S}\right)^{-1}\right]_{ij}\!\right),\nonumber
\end{equation}
built from the inverse of the observed part of the lag moments. We have empirically observed that a certain normalized version of these features exhibits nontrivial separability and stability properties across a broad range of connectivity regimes and noise correlation levels~$\beta$. Namely, coupled with the features $\mathcal{F}^{(n)}$, see equation~\eqref{eq:features}, they aid in mitigating the \emph{oscillation} and drift caused by the noise structure. More concretely, define $\mathcal{K}_{ij}^{(n)}=\left(\mathcal{F}_{ij}^{(n)},\mathcal{T}_{ij}^{(n)}\right)$ for each $ij$.  
Then, given Lemma~$2$ in~\cite{SMachado}, the identifiability gap separating the features of connected pairs from features of disconnected pairs (refer to its formal definition in~\cite{SMachado}) is larger for the features $\mathcal{K}^{(n)}$ than for either the features $\mathcal{T}^{(n)}$ or $\mathcal{F}^{(n)}$ w.h.p.~under the sufficient condition in Theorem~$2$. This tends to make the features $\mathcal{K}^{(n)}$ more amenable to generalization when using supervised methods to cluster them.


In summary, the proposed features lead to the appropriate identifiability properties if either the variability of the off-diagonal entries of $\Sigma_x$ is not too large (Theorem~\ref{th:affine}) or the exogeneous intervention is large enough (Remark~\ref{rem:ex}). 

\section{Methodology}
We adopt the following training set
\begin{equation} 
{\sf Tr}^{(n)}\overset{\Delta}=\left\{\left(\overline{\mathcal{K}}^{(n)}_{ij},\mathbf{1}_{\left\{A_{ij}\neq 0\right\}}\right)\right\}_{i\neq j}, 
\vspace{-5px}
\end{equation}
where we define the \emph{normalized} features  $\overline{\mathcal{K}}^{(n)}_{ij}:=\left[{\sf SC}\left(\mathcal{K}^{(n)}\right)\right]_{ij}$,
with ${\sf SC}$ denoting the \emph{standard-scaler}, i.e., the set~$\mathcal{K}^{(n)}=\left\{\mathcal{K}^{(n)}_{ij}\right\}_{i\neq j}$ is centered and the entries of each feature $\mathcal{K}^{(n)}_{ij}$ are normalized to meet unit variance as described in~\cite{scikit-learn}. The standard-scaler normalization helps to mitigate the drift of the features  caused by the average noise cross-correlation~$\beta$.
That is, for training, we provide a normalized set of features~$\overline{\mathcal{K}}^{(n)}$, with $\overline{\mathcal{K}}^{(n)}_{ij}$ being the feature of the pair $ij$ used as input to a Feed Forward Neural Network (FFNN) whose training output should be the ground truth connectivity of the pair $ij$, i.e., $1_{\left\{A_{ij}\neq 0\right\}}$.
\begin{figure*}[!t]
	\centering
\includegraphics[width=0.97\linewidth]{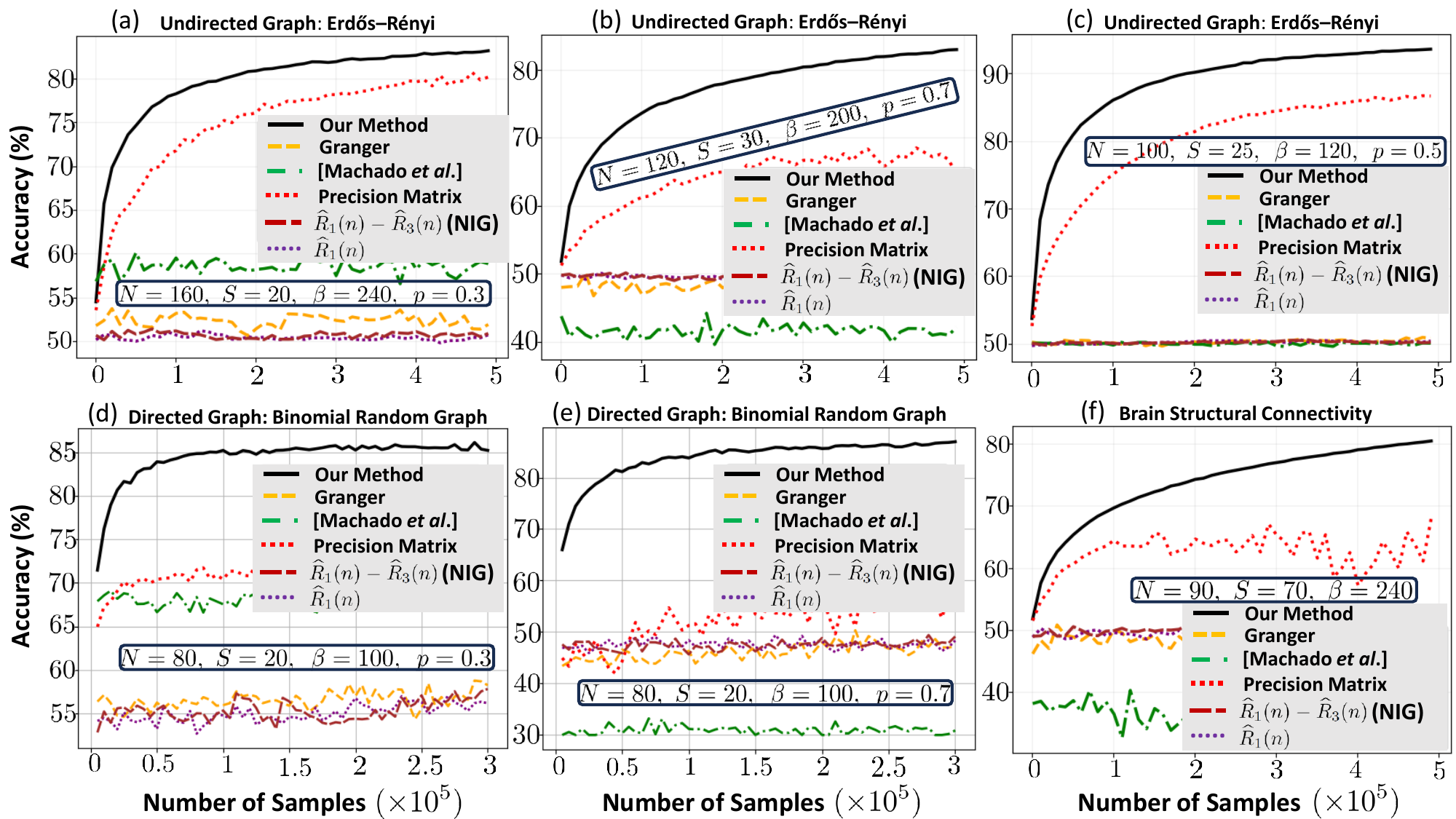}
	\caption{Plots (a)-(f) contrast the accuracy of distinct methods as a function of the number of time series samples: $i)$ (a)-(c) undirected networks (graph learning); (d)-(e) directed networks (causal inference); and (f) refers to a directed Brain connectome matrix. The hyperparameters $N$, $S$, $\beta$ and $p$ displayed framed in the plots entail the underlying regime.}
	\label{fig:Performance}
\end{figure*}


We generate the time series data~$\left\{\mathbf{y}(i)\right\}_{i=0}^{n}$ as done in~\cite{SMachado}, i.e., we first build a graph~$G$ on~$N$ nodes and then assign weights to the graph~$G$ to obtain an $N\times N$ interaction matrix $A$ that is nonnegative and stable, i.e., $\rho(A)<1$. Namely, we create~$G$ as a realization of an Erdős–Rényi random graph -- without self-loops, i.e., $G_{ii}=0$ for all $i$ -- with $N$ nodes and probability of edge or arrow placing~$p$. The value of~$p$  controls the sparsity of the underlying network. Then, the interaction matrix $A$ is characterized by: $A_{ij} = \alpha \frac{G_{ij}}{d_{\max}(G)}$, for $i\neq j$; and $A_{ii}  =  \rho-\sum_{k\neq i} A_{ik}$,  for all $i$ 
where $d_{\max}(G)$ is the maximum \emph{inwards}-degree of the underlying graph $G$ and $\alpha,\rho$ are stability hyperparameters with $0<\alpha\leq\rho<1$. This weight assignment is often referred to as \emph{Laplacian rule}~\cite{chungspectral}. This generative process renders the NDS~\eqref{eq:model} stable and with a support graph of interactions given by $G$ -- the main object of inference in our framework. To generate a directed graph $G$, we considered the realization of the Erdős–Rényi version for directed graphs also known as Binomial random graph model.

We train FFNNs over one realization of a random graph model with~$p=0.5$ and~$N=50$ and apply this trained FFNN to networked systems with distinct connectivities obtained with different $p$. Further, for training, we generate datasets with distinct average noise cross-correlation levels $\beta=0,5,10,\ldots,50$. Remarkably, our method generalizes to distinct connectivity regimes and different noise correlation levels as seen in the next Section. 



\section{Simulation Results}
\label{sec:simulation}

Fig.~\ref{fig:Performance} displays the performance of distinct structure estimation methods. The performance is captured by the \emph{accuracy} of the graph structure estimation as a function of the number of samples of the time series, known as empirical sample complexity. The accuracy is determined by the ratio of the number of \emph{directed} pairs correctly classified by each method over the total number of directed pairs comprising the graph.

We have chosen to compare our causal inference method with: $i)$ the recently proposed feature based approach~\cite{SMachado}; $ii)$ the Granger estimator~\cite{Geigeretal15,tomo_journal}; $iii)$ the one-lag $\widehat{R}_1(n)$ estimator~\cite{open_Journal}; iv) the NIG $\widehat{R}_1(n)-\widehat{R}_3(n)$ estimator~\cite{R1minusR3}; v) the Precision matrix or Graphical Lasso for sparse regimes (interchangeably)~\cite{Lasso}. 
While other graph learning or causal inference methods are optimal for specific regimes of observability or connectivity (mostly for \emph{very} sparse networks), the methods $ii)-v)$ we choose as benchmarks are popular and can be applied to partial observability (with technical guarantees, e.g., \cite{Mei_latent,tomo_journal_proceedings}) and across various connectivity graphs  (dense or sparse networks), thus, ideal to be contrasted with our method across distinct regimes of observability and connectivity. As in~\cite{R1minusR3}, we deploy Gaussian mixture on the estimators $ii)-v)$
 to recover the underlying graph from the estimated affinity matrix.  

Fig.~\ref{fig:Performance} demonstrates the overall performance superiority of our method in the colored regimes considered. Fig.~\ref{fig:Performance}~(f) refers to a human brain \emph{connectome}, obtained via tractography~\cite{HB}. In this latter case, we generate the time series~$\left\{\mathbf{y}(n)\right\}_{n=1}^M$ for this real world network to infer its connectivity. The regimes (a)-(c) refer to undirected graphs obtained as realizations of an Erdős–Rényi random graph model, and (d),(e) refer to directed graphs, with the networks drawn from a Binomial random graph model. 

%
%

The hyperparameters $N$, $S$, $\beta$ and $p$ represent the number of nodes in the NDS, number of observed nodes, average cross-correlation of the noise (average of the off-diagonal entries of $\Sigma_x$), and the probability of edge or arrow placing in the corresponding  Erdős–Rényi or Binomial random graph model (where applicable). Additional experiments for undirected graphs are in~\cite{ASantos_Icassp_2024}. 


\section{Concluding Remarks}
%
We studied the problem of network identification from the times series data streaming from a linear networked dynamical system (NDS). We considered the challenging framework of partial observability under colored noise. We first devised a novel identifiability condition for linear NDS under this regime. Then, departing from the scalar based methods that are dominant in causal inference, we considered a feature based paradigm: to each pair of nodes, we assign a feature vector. We studied the impact of the noise structure on the identifiability properties of the features of~\cite{SMachado} to propose novel features. We trained FFNNs with our features to cluster them and recover the causal network structure. The numerical experiments across distinct regimes of connectivity, observability and noise correlation demonstrate the competitive performance of our approach. 


\section{Acknowledgments}
This work is partially supported by the U.S. National Science Foundation under Grant CCF 2327905. The work of A. Santos was funded in part by FCT/MCTES through national funds and when applicable co-funded EU funds under the project UIDB/50008/2020 and project UIDP/00326/2020. José M. F. Moura was partially supported by the U.S. National Science Foundation under Grant CCN 1513936. D. Rente and R. Seabra were funded in part by FCT - Foundation for Science and Technology, Portugal, under project UIDP/00326/2020. The source code for the numerical experiments can be found at \url{https://github.com/seabrapt/brain_underlying_structure_identification}.

{
	\bibliographystyle{aaai24}
	\bibliography{biblio}
}

\end{document}